\documentclass{article}

\PassOptionsToPackage{numbers, compress}{natbib}

  \usepackage[preprint]{neurips_2022}

\usepackage[utf8]{inputenc} 
\usepackage[T1]{fontenc}    
\usepackage{hyperref}       
\usepackage{url}            
\usepackage{booktabs}       
\usepackage{amsfonts}       
\usepackage{nicefrac}       
\usepackage{microtype}      
\usepackage{xcolor}         

\usepackage{amsmath}
\usepackage{amssymb}
\usepackage{mathtools}
\usepackage{amsthm}
\usepackage{multirow, bm, bbding, bbm}
\usepackage{listings, amsmath, algorithm, graphicx, multirow}
\usepackage{wrapfig, algpseudocode}
\definecolor{codegreen}{rgb}{0,0.6,0}
\definecolor{codegray}{rgb}{0.5,0.5,0.5}
\definecolor{codepurple}{rgb}{0.58,0,0.82}
\definecolor{backcolour}{rgb}{0.95,0.95,0.92}

\lstdefinestyle{mystyle}{
backgroundcolor=\color{backcolour},
commentstyle=\color{codegreen},
keywordstyle=\color{magenta},
numberstyle=\tiny\color{codegray},
stringstyle=\color{codepurple},
basicstyle=\ttfamily\footnotesize,
breakatwhitespace=false,
breaklines=true,
captionpos=b,
keepspaces=true,
numbers=left,
numbersep=5pt,
showspaces=false,
showstringspaces=false,
showtabs=false,
tabsize=2
}
\theoremstyle{plain}
\newtheorem{theorem}{Theorem}

\theoremstyle{definition}

\newtheorem{assumption}{Assumption}

\newtheorem{remark}{Remark}

\title{Reasonable Effectiveness of Random Weighting:\\ A Litmus Test for Multi-Task Learning}

\author{%
Baijiong Lin$^1$, Feiyang Ye$^{1,2}$, Yu Zhang$^{1,4,}$\thanks{Corresponding author: Yu Zhang},~~Ivor W. Tsang$^3$ \\
$^1$ Department of Computer Science and Engineering, \\
Southern University of Science and Technology \\
$^2$ University of Technology Sydney\\
$^3$ Centre for Frontier AI Research (CFAR), \\
Agency for Science, Technology and Research (A$^*$STAR) \\
$^4$ Peng Cheng Laboratory\\
\texttt{bj.lin.email@gmail.com,12060007@mail.sustech.edu.cn}\\
\texttt{yu.zhang.ust@gmail.com,ivor\_tsang@ihpc.a-star.edu.sg} \\
}

\begin{document}

\maketitle

\begin{abstract}
Multi-Task Learning (MTL) has achieved success in various fields. However, how to balance different tasks to achieve good performance is a key problem. To achieve the task balancing, there are many works to carefully design dynamical loss/gradient weighting strategies but the basic random experiments are ignored to examine their effectiveness. In this paper, we propose the Random Weighting (RW) methods, including Random Loss Weighting (RLW) and Random Gradient Weighting (RGW), where an MTL model is trained with random loss/gradient weights sampled from a distribution. To show the effectiveness and necessity of RW methods, theoretically we analyze the convergence of RW and reveal that RW has a higher probability to escape local minima, resulting in better generalization ability. Empirically, we extensively evaluate the proposed RW methods to compare with twelve state-of-the-art methods on five image datasets and two multilingual problems from the XTREME benchmark to show RW methods can achieve comparable performance with state-of-the-art baselines. Therefore, we think that the RW methods are important baselines for MTL and should attract more attentions.
\end{abstract}

\section{Introduction}
\label{sec:intro}

Multi-Task Learning (MTL) \cite{caruana1997multitask, ZhangY21, Vandenhende21} aims to jointly train several related tasks to improve their generalization performance by leveraging common knowledge among them. Since MTL could not only significantly reduce the model size as well as speed up the inference but also improve the performance, it has been successfully applied to various fields \cite{ZhangY21}. However, when all the tasks are not highly related, which may be reflected via conflicting gradients or dominating gradients \cite{pcgrad}, it is more difficult to train an MTL model than training them separately because some tasks dominantly influence model parameters, leading to unsatisfactory performance for other tasks. This phenomenon is related to the \textit{task balancing} problem \cite{Vandenhende21} in MTL.



Recently, several works focus on tackling this issue from an optimization perspective via dynamically weighting task losses or balancing task gradients in the training process, called \textit{loss balancing} and \textit{gradient balancing} methods, respectively. However, all of the existing works take Equal Weighting (\textbf{EW}) which uses the fixed and equal weights in the whole training process as a basic baseline to test the effectiveness of their methods. We think that this baseline is not sufficient and it is quite necessary to conduct random experiments, which is missing in existing works, as a baseline to test them.

Therefore, in this paper, we propose the Random Weighting (\textbf{RW}) methods including Random Loss Weighting (\textbf{RLW}) and Random Gradient Weighting (\textbf{RGW}) as more reasonable baselines to test loss and gradient balancing methods, respectively. Specifically, in each training iteration, we first sample loss/gradient weights from a distribution with some possible normalization and then minimize the aggregated loss/gradient weighted by the random loss/gradient weights. Although the RW methods seem unreasonable, they can not only converge but also achieve comparable performance with existing methods that use carefully tuned weights. Thus, we think the RW methods are important baselines for MTL and deserve more attention.

To better understand the effectiveness and necessity of RW methods, we provide both theoretical analyses and empirical evaluations. Theoretically, we show RW methods are the stochastic variants of EW. From this perspective, we give a convergence analysis for RW methods. Besides, we can show that RW methods have a higher probability to escape local minima than EW, resulting in better generalization performance. Empirically, we investigate lots of State-Of-The-Art (SOTA) task balancing approaches including four loss balancing methods and eight gradient balancing methods. On five Computer Vision (CV) datasets and two multilingual problems from the XTREME benchmark \cite{hu20b}, we show that RW methods can consistently outperform EW and have competitive performance with existing SOTA methods.

In summary, the main contributions of this paper are three-fold.
\begin{itemize}


\item We propose the simple RW methods as novel baselines and litmus tests for MTL.

\item We provide the convergence guarantee and effectiveness analysis for RW methods.


\item Extensive experiments show that RW can outperform EW and achieve comparable performance with the SOTA methods.
\end{itemize}

\section{An Overview of Task Balancing Methods} \label{sec:preliminary}

\noindent\textbf{Notations.} Suppose there are $T$ tasks and task $t$ has its corresponding dataset $\mathcal{D}_t$. An MTL model usually contains two parts of parameters: task-sharing parameters $\theta$ and task-specific parameters $\{\psi_t\}_{t=1}^{T}$. For example, in CV, $\theta$ usually denotes parameters in the feature extractor shared by all the tasks and $\psi_t$ represents parameters in the task-specific output module for task $t$. Let $\ell_t(\mathcal{D}_t;\theta,\psi_t)$ denotes the average loss on $\mathcal{D}_t$ for task $t$. $\{\lambda_t\}_{t=1}^T$ are task-specific loss weights with a constraint that $\lambda_t^{l}\ge0$ for all $t$'s. Similarly, $\{\lambda_t^{g}\}_{t=1}^T$ denote task-specific gradient weights.

\noindent\textbf{Conventional Baseline with Fixed Weights.} Since there are multiple losses in MTL, they usually are aggregated as a single one via loss weights as
\begin{equation}
\mathcal{L}(\theta, \{\psi_t\}_{t=1}^T)=\sum_{t=1}^{T}\lambda_t^{l}\ell_t(\mathcal{D}_t;\theta,\psi_t).
\label{eq:mtl_loss}
\end{equation}
Apparently, the most simple method for loss weighting is to assign the same weight to all the tasks in the whole training process, i.e., without loss of generality, $\lambda_t^l=\frac{1}{T}$ for all $t$'s in every iteration. This approach is a common baseline in MTL and it is called EW in this paper. 

\noindent\textbf{Loss Balancing Methods.}
To achieve task balancing and improve the performance of MTL model, loss balancing methods aim to study how to generate appropriate loss weights $\{\lambda_t^l\}_{t=1}^T$ in Eq.~(\ref{eq:mtl_loss}) in every iteration and some representative methods include Uncertainty Weights (\textbf{UW}) \cite{kgc18}, Dynamic Weight Average (\textbf{DWA}) \cite{ljd19}, \textbf{IMTL-L} \cite{liu2021imtl} and Multi-Objective Meta Learning (\textbf{MOML}) \cite{ye2021multi}. These four methods focus on using higher loss weights for more difficult tasks measured by the uncertainty, learning speed, relative loss value, and validation performance, respectively. When minimizing Eq.~(\ref{eq:mtl_loss}), the learning rate of optimizing each task-specific parameter $\psi_t$ will be affected by the corresponding loss weight $\lambda_{t}^{l}$, which is the major difference between loss balancing and gradient balancing methods.

\noindent\textbf{Gradient Balancing Methods.} This type of methods think that the task balancing problem is caused by conflicting task gradients and the inappropriate gradient to update task-sharing parameters, thus they solve it via generating appropriate gradient weights $\{\lambda_t^g\}_{t=1}^T$ to balance the task gradients and make a better update of $\theta$ in every iteration as
\begin{equation}
\theta = \theta - \eta\sum_{t=1}^{T}\lambda_t^g\nabla_{\theta}\ell_t(\mathcal{D}_t;\theta,\psi_t).
\label{eq:mtl_grad}
\end{equation}
Noticeably, in such type methods, the gradient weights $\{\lambda_t^g\}_{t=1}^T$ only affect the task-sharing parameter $\theta$ but not task-specific parameters $\{\psi_t\}$, each of which is updated by the $t$-th task gradient $\nabla_{\psi_t}\ell_t(\mathcal{D}_t;\theta,\psi_t)$.

Some representative works include \textbf{MGDA-UB} \cite{sk18}, Gradient Normalization (\textbf{GradNorm}) \cite{chen2018gradnorm}, Projecting Conflicting Gradient (\textbf{PCGrad}) \cite{pcgrad}, Gradient sign Dropout (\textbf{GradDrop}) \cite{ChenNHLKCA20}, Impartial Multi-Task Learning (\textbf{IMTL-G}) \cite{liu2021imtl}, Gradient Vaccine (\textbf{GradVac}) \cite{wang2021gradient}, Conflict-Averse Gradient (\textbf{CAGrad}) \cite{liu2021conflict}, and \textbf{RotoGrad} \cite{javaloy2022rotograd}. Those eight methods focus on finding an aggregated gradient by linearly combining all the task gradients under different constraints such as equal gradient magnitude in GradNorm and equal gradient projection in IMTL-G to eliminate the gradient conflicting.

Compared with the EW method, those two types of methods use a dynamic weighting process where loss/gradient weights vary over training iterations or epochs. Thus, it is natural to think how about training an MTL model with random weights. Inspired by this, we propose the RW methods by randomly sampling loss/gradient weights in each iteration as the random experiments for loss/gradient balancing methods, respectively. Besides, we think RW methods are more reasonable baselines than EW as the litmus tests for MTL methods.



\section{The Random Weighting Methods} \label{sec:method}

In this section, we introduce the RW methods, including the RLW and RGW methods.

We focus on the update of task-sharing parameter $\theta$ as it is the key problem in MTL. In the following, we mainly introduce the RLW method as the RGW method acts similarly to the RLW method. For notation simplicity, we do not distinguish between ${\lambda}^l_t$ and ${\lambda}^g_t$ and denote them by ${\lambda_t}$. Besides, we denote $\bm{\ell}(\theta)=(\ell_1(\mathcal{D}_1;\theta, \psi_1), \cdots, \ell_T(\mathcal{D}_T;\theta, \psi_T))$, where the datasets $\{\mathcal{D}_t\}_{t=1}^T$ and the task-specific parameters $\{\psi_t\}_{t=1}^T$ are omitted for brevity.

Different from those loss balancing methods, RLW considers the loss weights $\bm{\lambda}=(\lambda_1,\cdots, \lambda_T)\in\mathbb{R}^T$ as random variables and samples them from a random distribution in each iteration. To guarantee loss weights in $\bm{\lambda}$ to be non-negative, we can first sample $\tilde{\bm{\lambda}}=(\tilde{\lambda}_1,\cdots, \tilde{\lambda}_T)$ from any distribution $p(\tilde{\bm{\lambda}})$ and then normalize $\tilde{\bm{\lambda}}$ into $\bm{\lambda}$ via a mapping $f$, where $f:\mathbb{R}^T\rightarrow\Delta^{T-1}$ is a normalization function such as the softmax function and $\Delta^{T-1}$ denotes a simplex in $\mathbb{R}^T$, i.e., $\bm{\lambda}\in\Delta^{T-1}$ means $\sum_{t=1}^T\lambda_t=1$ and $\lambda_t\ge0$ for all $t$. Note that $p(\bm{\lambda})$ is different from $p(\tilde{\bm{\lambda}})$ unless $f$ is an identity function. Finally, RLW updates the $\theta$ by computing the aggregated gradient $\nabla_{\theta} \bm{{\lambda}}^\top\bm{\ell}(\theta)$.

In this way, the RLW method uses dynamical loss weights in the training process, which is similar to existing loss balancing methods, but RLW uses random weights instead of carefully designed ones in the existing works. Therefore, RLW is a basic random experiment for those loss balancing methods to examine their effectiveness, which indicates RLW is a more reasonable baseline than the conventional EW.

\vspace{-0.2in}
\begin{minipage}{0.475\textwidth}
\begin{algorithm}[H]
\caption{A Training Iteration in RLW}
\label{alg:RLW}
\begin{algorithmic}[1]
\State {\bfseries Input:} numbers of tasks $T$, learning rate $\eta$, dataset $\{\mathcal{D}_t\}_{t=1}^T$, weight distribution $p(\tilde{\bm{\lambda}})$, normalization function $f$
\State {\bfseries Output:} task-sharing parameter $\theta^\prime$, task-specific parameters $\{\psi_t^\prime\}_{t=1}^T$
\For{$t=1$ {\bfseries to} $T$}
\State Compute loss $\ell_t({\mathcal{D}}_t;\theta, \psi_t)$;
\EndFor
\State Sample weights $\tilde{\bm{\lambda}}$ from $p(\tilde{\bm{\lambda}})$ and normalize it into ${\bm{\lambda}}$ via $f$;\Comment{RLW Method}
\State $\theta^\prime=\theta-\eta\nabla_{\theta}\sum_{t=1}^T{\lambda}_t\ell_t({\mathcal{D}}_t;\theta, \psi_t)$;
\For{$t=1$ {\bfseries to} $T$}
\State $\psi_t^\prime=\psi_t-\eta\nabla_{\psi_t}{\lambda}_t\ell_t({\mathcal{D}}_t;\theta, \psi_t)$;
\EndFor
\end{algorithmic}
\end{algorithm}
\end{minipage}
\hfill
\begin{minipage}{0.475\textwidth}
\begin{algorithm}[H]
\caption{A Training Iteration in RGW}
\label{alg:RGW}
\begin{algorithmic}[1]
\State {\bfseries Input:} numbers of tasks $T$, learning rate $\eta$, dataset $\{\mathcal{D}_t\}_{t=1}^T$, weight distribution $p(\tilde{\bm{\lambda}})$, normalization function $f$
\State {\bfseries Output:} task-sharing parameter $\theta^\prime$, task-specific parameters $\{\psi_t^\prime\}_{t=1}^T$
\For{$t=1$ {\bfseries to} $T$}
\State Compute loss $\ell_t({\mathcal{D}}_t;\theta, \psi_t)$;
\State Compute gradient $g_t=\nabla_{\theta}\ell_t$ or $\nabla_{z}\ell_t$;
\EndFor
\State Sample weights $\tilde{\bm{\lambda}}$ from $p(\tilde{\bm{\lambda}})$ and normalize it into ${\bm{\lambda}}$ via $f$;\Comment{RGW Method}
\State $\theta^\prime=\theta-\eta\nabla_{\theta}\sum_{t=1}^T{\lambda}_t g_t$;
\For{$t=1$ {\bfseries to} $T$}
\State $\psi_t^\prime=\psi_t-\eta\nabla_{\psi_t}\ell_t({\mathcal{D}}_t;\theta, \psi_t)$;
\EndFor
\end{algorithmic}
\end{algorithm}
\end{minipage}

Noticeably, the loss weights $\bm{\lambda}$ are random variables and vary over training iterations, thus it is apparently that the gradient $\nabla_{\theta} \bm{{\lambda}}^\top\bm{\ell}(\theta)$ of RLW is an unbiased estimation of the gradient $\mathbb{E}[\bm{\lambda}]^\top\nabla_{\theta}\bm{\ell}(\theta)$, where $\mathbb{E}[\bm{\lambda}]$ is the expectation of $\bm{\lambda}$ over the whole training process. This means that the RLW method is a stochastic variant of the loss balancing method with fixed weights $\mathbb{E}[\bm{\lambda}]$.
In particular, if $\mathbb{E}[\bm{\lambda}]$ is proportional to $(\frac{1}{T}, \cdots, \frac{1}{T})$, RLW is a stochastic variant of the conventional EW baseline. In Section \ref{sec:analysis}, we theoretically show that RLW has a better generalization performance that EW because of the extra randomness from loss weight sampling, which indicates the RLW method is a more effective baseline than EW.

Similar to RLW, in each iteration, RGW first randomly samples gradient weights $\tilde{\bm{\lambda}}$ from $p(\tilde{\bm{\lambda}})$, then normalizes it to obtain $\bm{\lambda}$ via $f$, and finally updates the task-sharing parameter $\theta$ by computing the aggregated gradient $\nabla_{\theta} \bm{{\lambda}}^\top\bm{\ell}(\theta)$. Following previous works \cite{sk18, ChenNHLKCA20, liu2021imtl, javaloy2022rotograd}, we compute the gradient with respect to the final hidden feature representation $z$ output from the shared parameter instead of the task-sharing parameter $\theta$ to reduce the computational cost. 
Thus, RGW is a random experiment for gradient balancing methods.

In this paper, we use the standard normal distribution for $p(\tilde{\bm{\lambda}})$ and the softmax function for $f$ in both the RLW and RGW methods since it is easy to implement, has a more stable performance (as shown in  experimental results in Section \ref{sec:robustness_on_distribution}), and is as efficient as the EW strategy (as shown in experimental results in Section \ref{sec:convergence_speed}). Besides, $\mathbb{E}[\bm{\lambda}]$ is proportional to $(\frac{1}{T}, \cdots, \frac{1}{T})$ as proved in Appendix~\ref{sec:mean}, thus it is fair to compare with the EW strategy.

The training algorithms of both RW methods are summarized in Algorithm \ref{alg:RLW} and \ref{alg:RGW}. The only difference between the RW methods and the existing works is the generation of loss/gradient weights (i.e., Line 6 in Algorithm \ref{alg:RLW} and Line 7 in Algorithm \ref{alg:RGW}). Apparently, the sampling operation in the RW methods is very easy to implement and only bring negligibly additional computational costs when compared with the existing works. Note that random weights are involved in the update of task-specific parameters in the RLW method but not the RGW method (i.e., Line 9 in Algorithm \ref{alg:RLW} and Line 10 in Algorithm \ref{alg:RGW}).


\section{Analysis}
\label{sec:analysis}

In this section, we analyze how the extra randomness from the loss/gradient weight sampling affects the convergence and effectiveness of the RW methods compared with the EW strategy.

We focus on the update of task-sharing parameter $\theta$ and take RLW as an example for analysis, which can easily be extended to the RGW method. For notation simplicity, we simply use $\ell_t(\theta)$ instead of $\ell_t(\mathcal{D}_t;\theta,\psi_t)$ to denote the loss function of task $t$ in this section and  Appendix~\ref{sec:analysis_appendix}. For ease of analysis, we make the following assumption.
\begin{assumption} \label{assumption}
$\mathbb{E}_{\mathcal{D}_t}[\| \nabla \ell_t(\mathcal{D}_t; \theta) \|^2]$ equals $\sigma_t^2$, the loss function $\ell_t(\theta)$ of task $t$ is $L_t$-Lipschitz continuous w.r.t. $\theta$, and $\bm{\lambda}$ satisfies $\mathbb{E}_{\bm{\lambda}}[\bm{\lambda}] = \bm{\mu}$.
\end{assumption}

In the following theorem, we analyze the convergence property of Algorithm \ref{alg:RLW} for the RLW method.
\begin{theorem} \label{thm:1}
Suppose the loss function $\ell_t(\theta)$ of task $t$ is $c_t$-strongly convex. We define $\theta_* = \arg\min_{\theta}\bm{\lambda}^\top\bm{\ell}(\theta)$ and denote by $\theta_k$ the solution in the $k$-th iteration. If $\eta$, the step size or equivalently the learning rate, satisfies $\eta\le 1 / 2c$, where $c = \min_{1 \le t \le T} \{c_t\}$, then under Assumption \ref{assumption} we have
\begin{equation}\label{thm:1eq}
\mathbb{E}[\| \theta_k - \theta_* \|^2] \le (1-2\eta c)^k \| \theta_0 - \theta_* \|^2 +\frac{\eta \kappa }{2c},
\end{equation}
where
$\kappa = \sum_{t=1}^T \sigma_t^2$. Then for any positive $\varepsilon$, $\mathbb{E} [\| \theta_k - \theta_* \|^2] \le \varepsilon $ can be achieved after $k = \frac{ \kappa}{2 \varepsilon c^2} \log \left( \frac{\varepsilon_0}{\varepsilon} \right)$ iterations with $\eta = \frac{\varepsilon c}{ \kappa}$, where $\varepsilon_0 =\mathbb{E} [\| \theta_0 - \theta_* \|^2] $.
\end{theorem}

Theorem \ref{thm:1} shows that the RLW method with a fixed step size has a linear convergence up to a radius around the optimal solution, which is similar to the EW strategy according to the property of the standard Stochastic Gradient Descent (SGD) method \cite{moulines2011non, NeedellSW16}. Although the RLW method has a larger $\kappa$ than the EW strategy, i.e., $\kappa_{\mathrm{EW}} = \sum_{t=1}^T \mu_t^2 \cdot \sum_{t=1}^T \sigma_t^2\le\kappa$, which may possibly require more iterations for the RLW method to reach the same accuracy as the EW strategy, experimental results in Section~\ref{sec:convergence_speed} show that empirically this does not cause much difference.

We next analyze the effectiveness of the RLW method from the perspective of stochastic optimization. It is observed that the SGD method can escape sharp local minima and converge to a better solution than Gradient Descent (GD) techniques under various settings with the help of noisy gradients \cite{hardt2016train, kleinberg2018alternative}. Inspired by those works, we prove Theorem \ref{thm:2} to show that the extra randomness in the RLW method can help RLW to better escape sharp local minima and achieve better generalization performance than the EW strategy.

Before presenting the theorem, for the ease of presentation, we introduction some notations. Here we consider the update step of these stochastic methods as $\theta_{k+1} = \theta_{k} - \eta(\nabla \bm{\mu}^\top\bm{\ell}(\theta_k) +\xi_k)$,  where $\xi_k$ is a noise with $\mathbb{E}[\xi_k]=0$ and $\|\xi_k\|^2 \le r$, and $r$ denotes the intensity of the noise. For the analysis, we construct an intermediate sequence $\varphi_k = \theta_{k} - \eta \nabla\bm{\mu}^\top \bm{\ell}(\theta_k)$. Then we get $\mathbb{E}_{\xi_k}[\varphi_{k+1}] = \varphi_k - \eta \nabla \mathbb{E}_{\xi_k}[\bm{\mu}^\top\bm{\ell}(\varphi_k - \eta \xi_k)]$. Therefore, the sequence $\{\varphi_k\}$ can be regarded as an approximation of using GD to minimize the function $\mathbb{E}_{\xi_k}[\bm{\mu}^\top\bm{\ell}(\varphi - \eta \xi_k)]$.

\begin{theorem} \label{thm:2}
Suppose $\nabla \ell_t(\theta)$ is $M_t$-Lipschitz continuous and $\|\xi_k\|^2 \le r$. If the loss function $\ell_t(\theta)$ of task $t$ is $c_t$-one point strongly convex w.r.t. a local minimum $\theta_*$ after convolved with noise $\xi$, i.e., $\left< \nabla \mathbb{E}_{\xi}\ell_t(\varphi - \eta \xi),\varphi - \theta_* \right> \ge c_t \|\varphi- \theta_* \|^2$, then under Assumption \ref{assumption}, after $K = \frac{1}{\rho}\log \left(\frac{\rho\varepsilon_0}{\beta}\right) $ iterations with $\eta \le \frac{c}{M^2}$, with probability at least $1-\delta$, we have $\|\varphi_K - \theta_* \|^2 \le  \frac{2\beta}{\rho\delta},$  where $\varepsilon_0 =\mathbb{E} [\| \varphi_0 - \theta_* \|^2] $, $c = \min_{1 \le t \le T} \{c_t\}$, $M = \max_{1 \le t \le T} \{M_t\}$, $\rho = 2\eta c - \eta^2M^2$, and $\beta = \eta^2r^2(1+\eta M)^2$.
\end{theorem}



Theorem~\ref{thm:2} only requires that $\ell_t(\theta)$ is $c_t$-one point strongly convex w.r.t. $\theta_*$ after convolved with noise $\xi$, which can hold for deep neural networks \cite{safran2021effects}. It also implies that for both RLW and EW methods, their solutions have high probabilities to be close to a local minimum $\theta_*$ depending on the noise $\xi$. Note that by adding extra noise, the sharp local minimum will disappear and only the flat local minimum with a large diameter will still exist \cite{kleinberg2018alternative}. On the other hand, those flat local minima could satisfy the one point strongly convexity assumption made in Theorem~\ref{thm:2}, thus the diameter of the converged flat local minimum is affected by the noise intensity.

\begin{remark}
Converging to flat local minima is important in neural network training because flat local minima may lead to better generalization \cite{chaudhari2019entropy, keskar2017large}. Due to the extra randomness from the sampling of loss weights, the RLW method can have a larger noise with a larger $r$ than the EW strategy (refer to Appendix \ref{appendix:noise}). Hence, according to Theorem~\ref{thm:2} and the above discussion, the RLW method can better escape sharp local minima and converge to a flatter local minimum than EW, resulting in better generalization performance.
\end{remark}

\section{Experiments}
In this section, we empirically evaluate the proposed RLW and RGW methods on five computer vision datasets (i.e., NYUv2, CityScapes, CelebA, Office-31, and Office-Home) and two multilingual problems from the XTREME benchmark \cite{hu20b}. All the experiments are conducted on one single NVIDIA GeForce RTX 3090 GPU. \textit{Due to page limit, experimental results on the CityScapes, CelebA, Office-31, and Office-Home datasets are put in Appendix \ref{sec:additional_results}.}

\noindent\textbf{Compared methods.} The baseline methods in comparison include several SOTA task balancing methods as introduced in Section \ref{sec:preliminary}, including four loss balancing methods (i.e., UW, DWA, IMTL-L, and MOML) and eight gradient balancing methods (i.e., MGDA-UB, GradNorm, PCGrad, GradDrop, IMTL-G, GradVac, CAGrad, and RotoGrad). For all the baseline methods, we directly use the optimal hyperparameters used in their original papers.

\noindent\textbf{Network architecture.} The network architecture we used adopts the Hard-Parameter Sharing (\textbf{HPS}) pattern \cite{Caruana93}, which shares bottom layers of the network for all the tasks and uses separate top layers for each task. Other MTL architectures are studied in Section \ref{sec:ablation_architecture}.

\noindent\textbf{Evaluation metric.}
For homogeneous MTL problems (e.g., the  XTREME benchmark and Office-31 dataset) which contain tasks of the same type such as the classification task, we directly use the average performance among tasks as the performance metric. For heterogeneous MTL problems (e.g., the NYUv2 dataset) that contain tasks of different types and may have multiple evaluation metrics for each task, by following \cite{ManinisRK19, Vandenhende21}, we use the average of the relative improvement over the EW method on each metric of each task as the performance measure, which is formulated as
\begin{equation*}
	\Delta_{\mathrm{p}}=100\%\times \frac{1}{T}\sum_{t=1}^T \frac{1}{N_t}\sum_{n=1}^{N_t}\frac{(-1)^{p_{t,n}}(M_{t,n}-M^{\mathrm{EW}}_{t,n})}{M^{\mathrm{EW}}_{t,n}},
\end{equation*}
where $N_t$ denotes the number of metrics in task $t$, $M_{t,n}$ denotes the performance of a task balancing method for the $n$th metric in task $t$, $M^{\mathrm{EW}}_{t,n}$ is defined similarly for the EW method, and $p_{t,n}$ is set to $1$ if a higher value indicates better performance for the $n$th metric in task $t$ and otherwise $0$.

\subsection{Results on the NYUv2 Dataset}

\noindent\textbf{Dataset.}
The NYUv2 dataset \cite{silberman2012indoor} is an indoor scene understanding dataset, which consists of video sequences recorded by the RGB and Depth cameras in the Microsoft Kinect. It contains 795 and 654 images for training and testing, respectively. This dataset includes three tasks: 13-class semantic segmentation, depth estimation, and surface normal prediction.

\noindent\textbf{Implementation details.}
For the NYUv2 dataset, the DeepLabV3+ architecture  \cite{ChenZPSA18} is used. Specifically, a ResNet-50 network pre-trained on the ImageNet dataset with dilated convolutions \cite{YuKF17} is used as a shared encoder among tasks and the Atrous Spatial Pyramid Pooling (ASPP) \cite{ChenZPSA18} module is used as the task-specific head for each task. Input images are resized to $288\times 384$. The Adam optimizer \cite{kingma2014adam} with the learning rate as $10^{-4}$ and the weight decay as $10^{-5}$ is used for training and the batch size is set to 8. We use the cross-entropy loss, $L_1$ loss, and cosine loss as the loss function of the semantic segmentation, depth estimation, and surface normal prediction tasks, respectively.

\noindent\textbf{Results.}
The results of different methods on the NYUv2 dataset are shown in Table~\ref{tbl:mtl-nyu_dmtl}. The top row shows the performance of the widely used EW strategy and we use it as a baseline to measure the relative improvement of different methods as shown in the definition of $\Delta_{\mathrm{p}}$. Rows 2-5 and 7-14 show the results of loss balancing and gradient balancing methods, respectively.

According to the results, we can see that both the RLW and RGW methods gain performance improvement over the EW strategy, which implies that training with extra randomness can have a better generalization ability. Besides, RLW has an improvement of $1.04\%$ over the EW strategy and it is higher than all loss balancing methods. As for gradient balancing methods, half of those methods have negligible or even negative improvement over the EW strategy and RGW can outperform five of them. Compared with all baselines, RLW is even higher than all of them except the CAGrad and RotoGrad methods, which indicates the random weights can easily beat the carefully designed ones.

According to the above analysis, there are two important conclusions. Firstly, the conventional EW strategy is a weaker baseline than RLW and RGW for MTL. Secondly, RW methods are competitive to SOTA methods and even performs better than some of them.



\begin{table*}[!h]
\centering
\caption{Performance on the \textbf{NYUv2} dataset with three tasks: 13-class semantic segmentation, depth estimation, and surface normal prediction. The best results for each task on each measure over loss/gradient balancing methods are marked with superscript $*$/$\dag$. The best results for each task on each measure over all methods are highlighted in \textbf{bold}. $\uparrow$ ($\downarrow$) indicates that the higher (lower) the result, the better the performance.}
\label{tbl:mtl-nyu_dmtl}
\resizebox{\textwidth}{!}{
\begin{tabular}{cccccccccccc}
\toprule
\multicolumn{2}{c}{\multirow{4}{*}{\textbf{Methods}}} & \multicolumn{2}{c}{\textbf{Segmentation}} & \multicolumn{2}{c}{\textbf{Depth}} & \multicolumn{5}{c}{\textbf{Surface Normal}} &
\multirow{4.5}{*}{\bm{$\Delta_{\mathrm{p}}$}${\uparrow}$}\\
\cmidrule(lr){3-4} \cmidrule(lr){5-6} \cmidrule(lr){7-11}
&& \multirow{2.5}{*}{\textbf{mIoU${\uparrow}$}} &  \multirow{2.5}{*}{\textbf{Pix Acc$\uparrow$}} &  \multirow{2.5}{*}{\textbf{Abs Err$\downarrow$}} &  \multirow{2.5}{*}{\textbf{Rel Err$\downarrow$}} & \multicolumn{2}{c}{\textbf{Angle Distance}} & \multicolumn{3}{c}{\textbf{Within $t^{\circ}$}} \\ \cmidrule(lr){7-8} \cmidrule(lr){9-11} & & & & & &\textbf{Mean$\downarrow$} & \textbf{Median$\downarrow$}  & \textbf{11.25$\uparrow$} & \textbf{22.5$\uparrow$} & \textbf{30$\uparrow$}  \\
\midrule
& EW & 53.77 & 75.45 & 0.3845 & 0.1605 & 23.5737 & 17.0438 & 35.04 & 60.93 & 72.07 & +0.00\%\\
\midrule
\multirow{5}{*}{\rotatebox{90}{Loss Bal.}} & UW & 54.14 & 75.92 & 0.3833 & 0.1597 & 23.2989 & 16.8691 & 35.33 & 61.37 & 72.48 & +0.64\%\\
& DWA & 53.81 & 75.56 & 0.3792$^*$ & 0.1565$^*$ & 23.6111 & 17.0609 & 34.89 & 60.89 & 71.97 & +0.63\%\\
& IMTL-L & 53.50 & 75.18 & 0.3824 & 0.1596 & 23.3805 & 16.8088 & 35.44 & 61.43 & 72.43 & +0.35\%\\
& MOML & \textbf{54.98}$^*$ & \textbf{75.98}$^*$ & 0.3877 & 0.1618 & 23.2401$^*$ & 16.7388 & 35.90$^*$ & 61.81$^*$ & 72.76$^*$ & +0.76\%\\
& \textbf{RLW (ours)} & 54.11 & 75.77 & 0.3809 & 0.1575 & 23.3777 & 16.7385$^*$ & 35.71 & 61.52 & 72.45 & +1.04\%$^*$ \\
\midrule
\multirow{9}{*}{\rotatebox{90}{Gradient Bal.}}  & MGDA-UB & 50.42 & 73.46 & 0.3834 & \textbf{0.1555}$^\dag$ & \textbf{22.7827}$^\dag$ & \textbf{16.1432}$^\dag$ & \textbf{36.90}$^\dag$ & 62.88 & 73.61 & +0.38\% \\
& GradNorm & 53.58 & 75.06 & 0.3931 & 0.1663 & 23.4360 & 16.9844 & 35.11 & 61.11 & 72.24 & -0.99\%\\
& PCGrad & 53.70 & 75.41 & 0.3903 & 0.1607 & 23.4281 & 16.9699 & 35.16 & 61.19 & 72.28 & -0.16\%\\
& GradDrop & 53.58 & 75.56 & 0.3855 & 0.1592 & 23.5518 & 17.0137 & 35.08 & 60.97 & 72.02 & +0.08\%\\
& IMTL-G & 53.54 & 75.45 & 0.3880 & 0.1589 & 23.0530 & 16.4328 & 36.21 & 62.31 & 73.15 & +0.80\% \\
& GradVac & 54.89$^\dag$ & \textbf{75.98}$^\dag$ & 0.3828 & 0.1635 & 23.6865 & 17.1301 & 34.82 & 60.71 & 71.81 & +0.07\%\\
& CAGrad & 53.12 & 75.19 & 0.3871 & 0.1599 & 22.5257 & 15.8821 & 37.42 & \textbf{63.50}$^\dag$ & \textbf{74.17}$^\dag$ & \textbf{+1.36\%}$^\dag$\\
& RotoGrad & 53.90 & 75.46 & 0.3812 & 0.1596 & 23.0197 & 16.3714 & 36.37 &  62.28 & 73.05 & +1.19\%\\
& \textbf{RGW (ours)} & 53.85 & 75.87 & \textbf{0.3772}$^\dag$ & 0.1562 & 23.6725 & 17.2439 & 34.62 & 60.49 & 71.75 & +0.62\%\\
\bottomrule
\end{tabular}}
\end{table*}



\subsection{Results on the XTREME benchmark}

\noindent\textbf{Dataset.}
The XTREME benchmark \cite{hu20b} is a large-scale multilingual multi-task benchmark for cross-lingual generalization evaluation, which covers fifty languages and contains nine tasks. We conduct experiments on two tasks containing Paraphrase Identification (PI) and Part-Of-Speech (POS) tagging in this benchmark. The datasets used in the PI and POS tasks are the PAWS-X dataset \cite{YangZTB19} and Universal Dependency v2.5 treebanks \cite{NivreMGHMPSTZ20}, respectively. On each task, we construct a multilingual problem by choosing the four languages with largest numbers of data, i.e., English (\texttt{en}), Mandarin (\texttt{zh}), German (\texttt{de}) and Spanish (\texttt{es}), for the PI task and English, Mandarin, Telugu (\texttt{te}) and Vietnamese (\texttt{vi}) for the POS task. The statistics for each language are summarized in Table~\ref{tab:nlp_data_count} in the Appendix. Different from the NYUv2 dataset where different tasks share the same input data, in those multilingual problems, each language/task has its own input data.

\noindent\textbf{Implementation details.}
For each multilingual problem in the XTREME benchmark, a pre-trained multilingual BERT (mBERT) model \cite{DevlinCLT19} implemented via the open-source transformers library \cite{WolfDSCDMCRLFDS20} is used as the shared encoder among languages and a fully connected layer is used as the language-specific output layer for each language. The Adam optimizer with the learning rate as $2\times 10^{-5}$ and the weight decay as $10^{-8}$ is used for training and the batch size is set to 32. The cross-entropy loss is used for the two multilingual problems.

\noindent\textbf{Results.}
According to experimental results shown in Table~\ref{tbl:mtl-multilingual_pi_pos}, we can find some empirical observations, which are similar to those on the NYUv2 dataset. Firstly, both the RLW and RGW strategies outperform the EW method. Secondly, compared with the existing works, RLW and RGW can achieve comparable performance with existing loss/gradient balancing methods, respectively. Even, RLW or RGW methods could outperform all baseline methods. For example, RLW achieves the best performance (i.e., 90.25\% average accuracy) on the PI problem and RGW achieves the best average F1 score of 91.16\% on the POS problem. It is interesting to find that the performance of RLW and RGW are inconsistent in different datasets. There is because the random loss weights in RLW will affect the update of task-specific parameters while not in RGW, which has a different influence on the performance of different datasets.

\begin{table*}[!h]
\caption{Performance on two multilingual problems, i.e., PI and POS from the \textbf{XTREME benchmark}. The best results for each language over loss/gradient balancing methods are marked with superscript $*$/$\dag$. The best results for each language over all methods are highlighted in \textbf{bold}.}
\label{tbl:mtl-multilingual_pi_pos}
\centering
\resizebox{\linewidth}{!}{
\begin{tabular}{cc|ccccc|ccccc}
\toprule
\multicolumn{2}{c|}{\multirow{2.5}{*}{\textbf{Methods}}} & \multicolumn{5}{c|}{\textbf{PI (Accuracy)}} & \multicolumn{5}{c}{\textbf{POS (F1 Score)}} \\
\cmidrule(r){3-7} \cmidrule(l){8-12}
&&{\texttt{en}} & {\texttt{zh}} & {\texttt{de}} & {\texttt{es}} & {\textbf{Avg}} &{\texttt{en}}&{\texttt{zh}}&{\texttt{te}}&{\texttt{vi}}&{\textbf{Avg}} \\
\midrule
&EW & 94.29 & 84.99 & 89.79 & 90.94 & 90.00 & 95.06 & 89.01 & 91.41 & 86.65 & 90.53\\
\midrule
\multirow{5}{*}{\rotatebox{90}{Loss Bal.}} & UW & 93.74 & 85.44$^*$ & \textbf{90.24}$^*$ & 91.29 & 90.18 & 94.89 & 88.77 & 90.96 & 87.12 & 90.44\\
& DWA & \textbf{94.69}$^*$ & 84.99 & 89.49 & \textbf{91.44}$^*$ & 90.15 & 95.02 & 89.03 & 91.87 & \textbf{87.27}$^*$ & 90.80 \\
& IMTL-L & 93.94 & 84.54 & 89.39 & \textbf{91.44}$^*$ & 89.82 & \textbf{95.57}$^*$ & 89.93$^*$ & 91.77 & 86.11 & 90.84\\
& MOML & 93.89 & 83.74 & 89.94 & 90.99 & 89.64 & {95.15} & {89.11} & 92.41 & 87.24 & {90.98}$^*$ \\
& \textbf{RLW (ours)} & 94.29 & 85.39 & 89.94 & 91.39 & \textbf{90.25}$^*$ & 95.01 & 88.87 & \textbf{92.86}$^*$ & 86.85 & 90.90 \\
\midrule
\multirow{9}{*}{\rotatebox{90}{Gradient Bal.}} & MGDA-UB & 94.09 & 84.14 & 89.14 & 90.59 & 89.49 & 94.89 & 88.43 & 91.01 & 86.04 & 90.01\\
& GradNorm & 94.19 & 83.59 & 88.89 & 91.24 & 89.47 & 94.88 & 88.80 & 91.78 & 86.96 & 90.61\\
& PCGrad & 94.19 & \textbf{85.49}$^\dag$ & 89.09 & 91.24 & 90.00 & 94.85 & 88.42 & 90.72 & 86.71 & 90.18 \\
& GradDrop & 94.29 & 84.44 & 89.69 & 90.94 & 89.84 & 95.08 & 89.06 & 90.65 & 87.17 & 90.49 \\
& IMTL-G & \textbf{94.69}$^\dag$ & 84.54 & 89.39 & 90.69 & 89.82 & 94.93 & 88.70 & 91.66 & 87.00 & 90.57\\
& GradVac & 94.29 & 84.94 & 89.19 & 90.89 & 89.83 & 94.87 & 88.41 & 90.62 & 86.47 & 90.09 \\
& CAGrad & 94.34 & 84.59 & 90.09$^\dag$ & 90.64 & 89.91 & 94.83 & 88.65 & 91.71 & 86.76 & 90.48\\
& RotoGrad & 93.99 & 83.89 & 89.29 & 90.94 & 89.52 & 95.44 & 89.79 & 91.42 & 86.33 & 90.74\\
& \textbf{RGW (ours)} & 94.55 & 84.99 & 89.29 & {91.40}$^\dag$ & 90.06$^\dag$ & 95.52$^\dag$ & \textbf{90.13}$^\dag$ & 91.82$^\dag$ & 87.18$^\dag$ & \textbf{91.16}$^\dag$\\
\bottomrule
\end{tabular}}
\end{table*}

\subsection{Robustness on Distribution} \label{sec:robustness_on_distribution}

In this section, we evaluate the robustness of the proposed RW methods on the sampling distribution. Taking RLW as an example, we show its robustness by evaluating with five different sampling distributions (i.e., $p(\tilde{\bm{\lambda}})$) for loss weights. The five distributions are uniform distribution between $0$ and $1$ (denoted by \textbf{Uniform}), standard normal distribution (denoted by \textbf{Normal}), Dirichlet distribution with $\alpha=1$ (denoted by \textbf{Dirichlet}), Bernoulli distribution with probability $1/2$ (denoted by \textbf{Bernoulli}), Bernoulli distribution with probability $1/2$ and a constraint $\sum_{t=1}^T\tilde{\lambda}_t=1$ (denoted by \textbf{c-Bernoulli}). We set $f$ as a function of $f(\tilde{\bm{\lambda}})=\tilde{\bm{\lambda}}/(\sum_{t=1}^T\tilde{\lambda}_t)$ for the Bernoulli distribution and the c-Bernoulli distribution, a softmax function for the Normal distribution and Uniform distribution, and an identity function for the Dirichlet distribution. We can prove that all the $\mathbb{E}[\bm{\lambda}]$'s under these five distributions equal  $(\frac{1}{T}, \cdots, \frac{1}{T})$ (refer to Appendix~\ref{sec:mean}), thus it is fair to compare among them.

Figure~\ref{fig:distribution} shows the results of the RLW method with five sampling distributions on the NYUv2 dataset in terms of $\Delta_{\mathrm{p}}$, where the experiment on each sampling distribution is repeated for 8 times. The results show that the RLW method with different distributions can always outperform the EW method, which shows the robustness of the RLW method with respect to the sampling distribution. In addition, compared with the uniform, Dirichlet, and Bernoulli distributions, RLW with the standard normal distribution achieves better and more stable performance. Although RLW with the c-Bernoulli distribution performs slightly better than the standard normal distribution, it is more unstable and may need a longer training time as shown in Section \ref{sec:convergence_speed}. Thus, in this paper, we use the standard normal distribution to sample loss weights.

\begin{wrapfigure}{r}{0.5\textwidth}
\vskip -0.4in
\centering
\includegraphics[width=\linewidth]{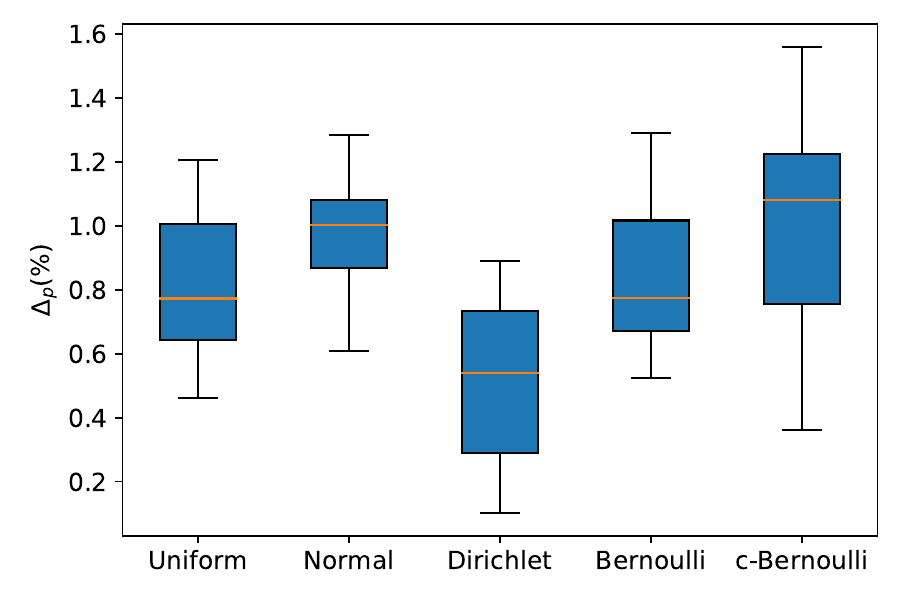}
\caption{Results of the RLW method with different sampling distributions in terms of $\Delta_{\mathrm{p}}$.}
\label{fig:distribution}
\end{wrapfigure}

\subsection{Convergence Speed} \label{sec:convergence_speed}

Here we take RLW as an example to show the efficiency of RW methods. Figure~\ref{fig:convergence} plots the performance curve on both NYUv2 and CelebA validation datasets to empirically compare the convergence speed of the EW and RLW methods.

On the NYUv2 dataset with three tasks, the performance curves of the RLW method with two sampling distributions are similar to that of the EW method, which indicates that the RLW method has a similar convergence property to the EW method on this dataset. As the number of tasks increases, i.e., on the CelebA dataset with 40 tasks, we find that the RLW method with the standard normal distribution still converges as fast as the EW method, while the RLW method with the c-Bernoulli distribution converges slower. One reason for this phenomenon is that only one task is used to update model parameters in each training iteration when using the c-Bernoulli distribution. Thus, in this paper, we use the standard normal distribution, which is as efficient as the EW method.

\begin{figure*}[!htbp]
	\centering
	\includegraphics[width=0.49\linewidth]{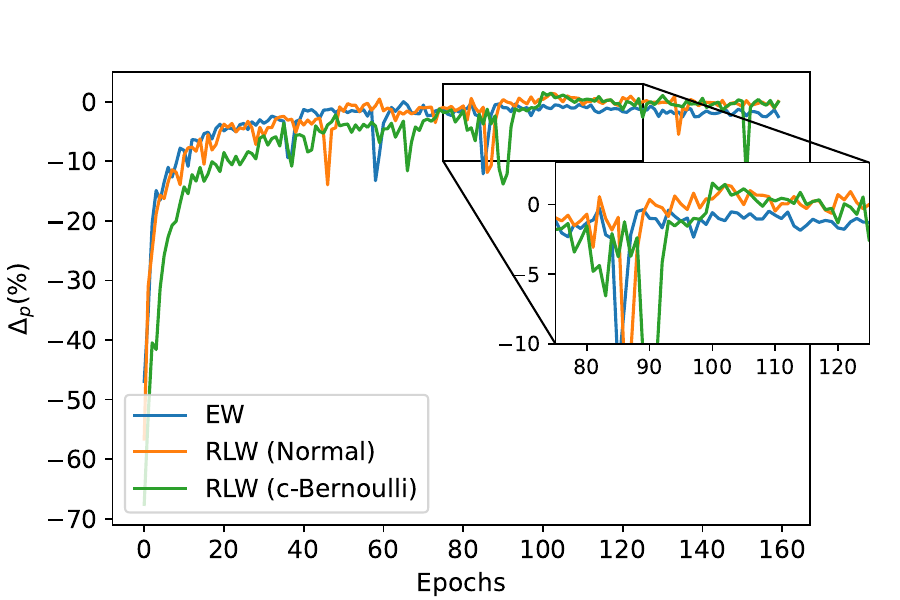}
	\hfill
	\includegraphics[width=0.49\linewidth]{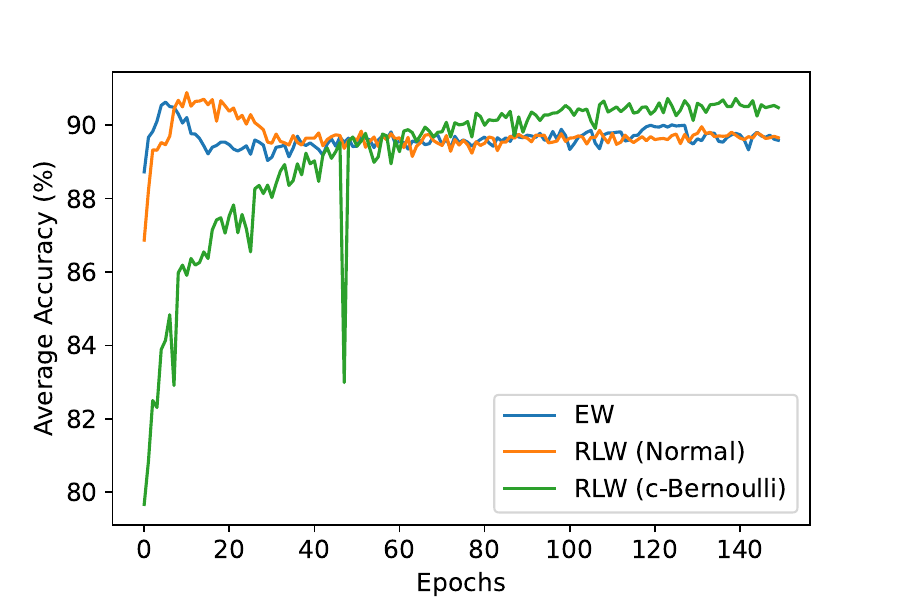}
	\caption{Comparison on the convergence speed of the EW and RLW methods on the NYUv2 validation dataset (\textbf{Left}) and the CelebA validation dataset (\textbf{Right}).}
	\label{fig:convergence}
\end{figure*}

\subsection{Combination of Loss and Gradient Balancing Methods} \label{sec:effectiveness_of_randomness}

The loss balancing methods are complementary with the gradient balancing methods. Following \cite{liu2021imtl}, we train an MTL model with different combinations of loss balancing and gradient balancing methods on the NYUv2 dataset to further improve the performance. We use the vanilla EW as the baseline to measure the relative improvement of the other different combinations as shown in the definition of $\Delta_{\mathrm{p}}$.

According to the results shown in Table~\ref{tbl:combinations}, we can see that combined with the UW, DWA and IMTL-L methods, some gradient balancing methods performs better but others become worse. For example, $\Delta_{\mathrm{p}}$ of the GradDrop method drops from $0.08\%$ to $-0.42\%$ when combined with DWA. Noticeably, by combining with the proposed RLW method, all the gradient balancing methods can achieve better performance. Besides, on each gradient balancing method, the improvement induced by the RLW method is significantly larger than the other three loss balancing methods as well as the EW method. Moreover, RGW can also improve the performance of loss balancing methods except DWA. Thus, this experiment further demonstrates the effectiveness of the proposed RW methods. 
\begin{table}[!h]
\centering
\caption{Results of different combinations of loss balancing and gradient balancing methods on the NYUv2 dataset in terms of $\Delta_{\mathrm{p}}$. The best results in each row are highlighted in \textbf{bold}.}
\label{tbl:combinations}
\begin{tabular}{cccccc}
\toprule
& EW & UW & DWA & IMTL-L &\textbf{RLW} \\
\midrule
Vanilla & +0.00\% & +0.64\% & +0.63\% & +0.35\% &\textbf{+1.04\%} \\
MGDA-UB & +0.38\% & +0.15\% & +0.47\% & -0.59\% &\textbf{+2.01\%}\\
GradNorm & -0.99\% & +0.87\% & -0.95\% & +0.54\% &\textbf{+0.89\%} \\
PCGrad & -0.16\% & +0.72\% & +0.19\% & +0.38\% &\textbf{+0.97\%} \\
GradDrop & +0.08\% & +0.25\% & -0.42\% & +0.36\% &\textbf{+0.93\%}\\
IMTL-G & +0.80\% & +0.45\% & +1.20\% & +0.18\% &\textbf{+1.50\%}\\
GradVac & +0.07\% & -0.03\% & +0.89\% & +0.69\% &\textbf{+0.97\%}\\
CAGrad & +1.36\% & +1.07\% & +1.41\% & +2.18\% &\textbf{+2.20\%}\\
RotoGrad & +1.19\% & +1.03\% & +0.75\% & +1.40\% & \textbf{+1.45\%}\\
\textbf{RGW} & +0.62\% & +0.82\% & +0.41\% & +0.78\% & \textbf{+1.46\%}\\
\bottomrule
\end{tabular}
\end{table}

\subsection{Effects of Different Architectures} \label{sec:ablation_architecture}

The proposed RW methods can be seamlessly incorporated into all the MTL architectures. To see this, we take RLW as an example and combine it with three different MTL architectures, i.e., \textbf{cross-stitch} network \cite{MisraSGH16}, Multi-Task Attention Network (\textbf{MTAN}) \cite{ljd19}, and \textbf{NDDR-CNN} \cite{gao2019nddr}. We use the combination of EW and HPS as the baseline to measure the relative improvement of the other different combinations as shown in the definition of $\Delta_{\mathrm{p}}$.

According to the results on the NYUv2 dataset as shown in Table~\ref{tbl:architectures}, we can see that the proposed RLW strategy outperforms the EW method under all the three architectures. When using the MTAN and NNDR-CNN architectures, RLW achieves better performance than the CAGrad method that performs best in the HPS architecture, which shows the potential of the proposed RLW method when choosing suitable MTL architectures. Moreover, combined with the RLW method, CAGrad can be further improved under the four architectures. For example, the combinations of RLW and CAGrad can achieve the best $\Delta_{\mathrm{p}}$ of $3.53\%$ under the NDDR-CNN architecture.

\begin{table}[!htbp]
\centering
\caption{Results of different combinations of task balancing methods and MTL architectures on the NYUv2 dataset in terms of $\Delta_{\mathrm{p}}$. The best results for each architecture are highlighted in \textbf{bold}.}
\label{tbl:architectures}
\begin{tabular}{ccccc}
\toprule
& HPS & Cross-stitch & MTAN & NDDR-CNN \\
\midrule
EW & +0.00\% & +1.43\% & +2.56\% & +1.90\% \\
CAGrad & +1.36\% & +2.42\% & +2.26\% & +2.83\%\\
\textbf{RLW} & +1.04\% & +2.23\% & {+2.66\%} & +2.91\%\\
\textbf{RLW}+CAGrad & \textbf{+2.20\%} & \textbf{+2.76\%} & \textbf{+2.92\%} & \textbf{+3.53\%}\\
\bottomrule
\end{tabular}
\end{table}

\section{Conclusions}

In this paper, we propose the RW methods, an important yet ignored baselines for MTL, by training an MTL model with random loss/gradient weights. We analyze the convergence and effectiveness properties of the proposed RW method. Moreover, we provide a consistent and comparative comparison to show the RW methods can achieve comparable performance with state-of-the-art methods that use carefully designed weights, which indicates the random experiments could be used to examine the effectiveness of newly proposed MTL methods and RW methods should attract wide attention as the litmus tests. In our future work, we will apply the RW methods to more MTL applications.

{\small
\bibliography{RLW}

\begin{thebibliography}{10}

\bibitem{Caruana93}
Rich Caruana.
\newblock Multitask learning: {A} knowledge-based source of inductive bias.
\newblock In {\em Proceedings of the 10th International Conference on Machine
  Learning}, pages 41--48, 1993.

\bibitem{caruana1997multitask}
Rich Caruana.
\newblock Multitask learning.
\newblock {\em Machine learning}, 28(1):41--75, 1997.

\bibitem{chaudhari2019entropy}
Pratik Chaudhari, Anna Choromanska, Stefano Soatto, Yann LeCun, Carlo Baldassi,
  Christian Borgs, Jennifer Chayes, Levent Sagun, and Riccardo Zecchina.
\newblock Entropy-sgd: Biasing gradient descent into wide valleys.
\newblock {\em Journal of Statistical Mechanics: Theory and Experiment},
  2019(12):124018, 2019.

\bibitem{ChenZPSA18}
Liang{-}Chieh Chen, Yukun Zhu, George Papandreou, Florian Schroff, and Hartwig
  Adam.
\newblock Encoder-decoder with atrous separable convolution for semantic image
  segmentation.
\newblock In {\em Proceedings of the 14th European Conference on Computer
  Vision}, volume 11211, pages 833--851, 2018.

\bibitem{chen2018gradnorm}
Zhao Chen, Vijay Badrinarayanan, Chen-Yu Lee, and Andrew Rabinovich.
\newblock Gradnorm: Gradient normalization for adaptive loss balancing in deep
  multitask networks.
\newblock In {\em Proceedings of the International Conference on Machine
  Learning}, pages 794--803. PMLR, 2018.

\bibitem{ChenNHLKCA20}
Zhao Chen, Jiquan Ngiam, Yanping Huang, Thang Luong, Henrik Kretzschmar, Yuning
  Chai, and Dragomir Anguelov.
\newblock Just pick a sign: Optimizing deep multitask models with gradient sign
  dropout.
\newblock In {\em Proceedings of the 33rd Advances in Neural Information
  Processing Systems}, 2020.

\bibitem{CordtsORREBFRS16}
Marius Cordts, Mohamed Omran, Sebastian Ramos, Timo Rehfeld, Markus Enzweiler,
  Rodrigo Benenson, Uwe Franke, Stefan Roth, and Bernt Schiele.
\newblock The cityscapes dataset for semantic urban scene understanding.
\newblock In {\em Proceedings of {IEEE} Conference on Computer Vision and
  Pattern Recognition}, pages 3213--3223, 2016.

\bibitem{DevlinCLT19}
Jacob Devlin, Ming{-}Wei Chang, Kenton Lee, and Kristina Toutanova.
\newblock {BERT:} pre-training of deep bidirectional transformers for language
  understanding.
\newblock In {\em Proceedings of the 2019 Conference of the North American
  Chapter of the Association for Computational Linguistics: Human Language
  Technologies}, pages 4171--4186, 2019.

\bibitem{gao2019nddr}
Yuan Gao, Jiayi Ma, Mingbo Zhao, Wei Liu, and Alan~L Yuille.
\newblock Nddr-cnn: Layerwise feature fusing in multi-task cnns by neural
  discriminative dimensionality reduction.
\newblock In {\em Proceedings of the IEEE/CVF Conference on Computer Vision and
  Pattern Recognition}, pages 3205--3214, 2019.

\bibitem{hardt2016train}
Moritz Hardt, Ben Recht, and Yoram Singer.
\newblock Train faster, generalize better: Stability of stochastic gradient
  descent.
\newblock In {\em Proceedings of the International Conference on Machine
  Learning}, pages 1225--1234. PMLR, 2016.

\bibitem{hu20b}
Junjie Hu, Sebastian Ruder, Aditya Siddhant, Graham Neubig, Orhan Firat, and
  Melvin Johnson.
\newblock {XTREME}: A massively multilingual multi-task benchmark for
  evaluating cross-lingual generalisation.
\newblock In {\em Proceedings of the 37th International Conference on Machine
  Learning}, volume 119, pages 4411--4421. PMLR, 2020.

\bibitem{javaloy2022rotograd}
Adri{\'a}n Javaloy and Isabel Valera.
\newblock Rotograd: Gradient homogenization in multitask learning.
\newblock In {\em Proceedings of the 10th International Conference on Learning
  Representations}, 2022.

\bibitem{kgc18}
Alex Kendall, Yarin Gal, and Roberto Cipolla.
\newblock Multi-task learning using uncertainty to weigh losses for scene
  geometry and semantics.
\newblock In {\em Proceedings of {IEEE} Conference on Computer Vision and
  Pattern Recognition}, pages 7482--7491, 2018.

\bibitem{keskar2017large}
Nitish~Shirish Keskar, Jorge Nocedal, Ping Tak~Peter Tang, Dheevatsa Mudigere,
  and Mikhail Smelyanskiy.
\newblock On large-batch training for deep learning: Generalization gap and
  sharp minima.
\newblock In {\em Proceedings of the 5th International Conference on Learning
  Representations}, 2017.

\bibitem{kingma2014adam}
Diederik~P. Kingma and Jimmy Ba.
\newblock Adam: {A} method for stochastic optimization.
\newblock In {\em Proceedings of the 3rd International Conference on Learning
  Representations}, 2015.

\bibitem{kleinberg2018alternative}
Bobby Kleinberg, Yuanzhi Li, and Yang Yuan.
\newblock An alternative view: When does sgd escape local minima?
\newblock In {\em Proceedings of the International Conference on Machine
  Learning}, pages 2698--2707. PMLR, 2018.

\bibitem{liu2021conflict}
Bo~Liu, Xingchao Liu, Xiaojie Jin, Peter Stone, and Qiang Liu.
\newblock Conflict-averse gradient descent for multi-task learning.
\newblock In {\em Proceedings of the 35th Advances in Neural Information
  Processing Systems}, 2021.

\bibitem{liu2021imtl}
Liyang Liu, Yi~Li, Zhanghui Kuang, Jing-Hao Xue, Yimin Chen, Wenming Yang,
  Qingmin Liao, and Wayne Zhang.
\newblock Towards impartial multi-task learning.
\newblock In {\em Proceedings of the 9th International Conference on Learning
  Representations}, 2021.

\bibitem{ljd19}
Shikun Liu, Edward Johns, and Andrew~J. Davison.
\newblock End-to-end multi-task learning with attention.
\newblock In {\em Proceedings of {IEEE} Conference on Computer Vision and
  Pattern Recognition}, pages 1871--1880, 2019.

\bibitem{liu2015faceattributes}
Ziwei Liu, Ping Luo, Xiaogang Wang, and Xiaoou Tang.
\newblock Deep learning face attributes in the wild.
\newblock In {\em Proceedings of International Conference on Computer Vision},
  2015.

\bibitem{ManinisRK19}
Kevis{-}Kokitsi Maninis, Ilija Radosavovic, and Iasonas Kokkinos.
\newblock Attentive single-tasking of multiple tasks.
\newblock In {\em Proceedings of {IEEE} Conference on Computer Vision and
  Pattern Recognition}, pages 1851--1860, 2019.

\bibitem{MisraSGH16}
Ishan Misra, Abhinav Shrivastava, Abhinav Gupta, and Martial Hebert.
\newblock Cross-stitch networks for multi-task learning.
\newblock In {\em Proceedings of the {IEEE} Conference on Computer Vision and
  Pattern Recognition}, pages 3994--4003, 2016.

\bibitem{moulines2011non}
Eric Moulines and Francis Bach.
\newblock Non-asymptotic analysis of stochastic approximation algorithms for
  machine learning.
\newblock In {\em Proceedings of the Advances in Neural Information Processing
  Systems}, volume~24, pages 451--459, 2011.

\bibitem{NeedellSW16}
Deanna Needell, Nathan Srebro, and Rachel Ward.
\newblock Stochastic gradient descent, weighted sampling, and the randomized
  kaczmarz algorithm.
\newblock {\em Mathematical Programming}, 155(1-2):549--573, 2016.

\bibitem{NivreMGHMPSTZ20}
Joakim Nivre, Marie{-}Catherine de~Marneffe, Filip Ginter, Jan Hajic,
  Christopher~D. Manning, Sampo Pyysalo, Sebastian Schuster, Francis~M. Tyers,
  and Daniel Zeman.
\newblock Universal dependencies v2: An evergrowing multilingual treebank
  collection.
\newblock In {\em Proceedings of the 12th Language Resources and Evaluation
  Conference}, pages 4034--4043, 2020.

\bibitem{saenko2010adapting}
Kate Saenko, Brian Kulis, Mario Fritz, and Trevor Darrell.
\newblock Adapting visual category models to new domains.
\newblock In {\em Proceedings of the 6th European Conference on Computer
  Vision}, pages 213--226. Springer, 2010.

\bibitem{safran2021effects}
Itay~M Safran, Gilad Yehudai, and Ohad Shamir.
\newblock The effects of mild over-parameterization on the optimization
  landscape of shallow relu neural networks.
\newblock In {\em Conference on Learning Theory}, pages 3889--3934. PMLR, 2021.

\bibitem{sk18}
Ozan Sener and Vladlen Koltun.
\newblock Multi-task learning as multi-objective optimization.
\newblock In {\em Proceedings of the 31st Advances in Neural Information
  Processing Systems}, pages 525--536, 2018.

\bibitem{silberman2012indoor}
Nathan Silberman, Derek Hoiem, Pushmeet Kohli, and Rob Fergus.
\newblock Indoor segmentation and support inference from rgbd images.
\newblock In {\em Proceedings of the 8th European Conference on Computer
  Vision}, pages 746--760, 2012.

\bibitem{Vandenhende21}
Simon Vandenhende, Stamatios Georgoulis, Wouter~Van Gansbeke, Marc Proesmans,
  Dengxin Dai, and Luc~Van Gool.
\newblock Multi-task learning for dense prediction tasks: A survey.
\newblock {\em {IEEE} Transactions on Pattern Analysis and Machine
  Intelligence}, 2021.

\bibitem{venkateswara2017deep}
Hemanth Venkateswara, Jose Eusebio, Shayok Chakraborty, and Sethuraman
  Panchanathan.
\newblock Deep hashing network for unsupervised domain adaptation.
\newblock In {\em Proceedings of the IEEE Conference on Computer Vision and
  Pattern Recognition}, pages 5018--5027, 2017.

\bibitem{wang2021gradient}
Zirui Wang, Yulia Tsvetkov, Orhan Firat, and Yuan Cao.
\newblock Gradient vaccine: Investigating and improving multi-task optimization
  in massively multilingual models.
\newblock In {\em Proceedings of the 9th International Conference on Learning
  Representations}, 2021.

\bibitem{WolfDSCDMCRLFDS20}
Thomas Wolf, Lysandre Debut, Victor Sanh, Julien Chaumond, Clement Delangue,
  Anthony Moi, Pierric Cistac, Tim Rault, R{\'{e}}mi Louf, Morgan Funtowicz,
  Joe Davison, Sam Shleifer, Patrick von Platen, Clara Ma, Yacine Jernite,
  Julien Plu, Canwen Xu, Teven~Le Scao, Sylvain Gugger, Mariama Drame, Quentin
  Lhoest, and Alexander~M. Rush.
\newblock Transformers: State-of-the-art natural language processing.
\newblock In {\em Proceedings of the 2020 Conference on Empirical Methods in
  Natural Language Processing}, pages 38--45, 2020.

\bibitem{YangZTB19}
Yinfei Yang, Yuan Zhang, Chris Tar, and Jason Baldridge.
\newblock {PAWS-X:} {A} cross-lingual adversarial dataset for paraphrase
  identification.
\newblock In {\em Proceedings of the 2019 Conference on Empirical Methods in
  Natural Language Processing and the 9th International Joint Conference on
  Natural Language Processing}, pages 3685--3690, 2019.

\bibitem{ye2021multi}
Feiyang Ye, Baijiong Lin, Zhixiong Yue, Pengxin Guo, Qiao Xiao, and Yu~Zhang.
\newblock Multi-objective meta learning.
\newblock In {\em Proceedings of the 35th Advances in Neural Information
  Processing Systems}, 2021.

\bibitem{YuKF17}
Fisher Yu, Vladlen Koltun, and Thomas~A. Funkhouser.
\newblock Dilated residual networks.
\newblock In {\em Proceedings of the {IEEE} Conference on Computer Vision and
  Pattern Recognition}, pages 636--644, 2017.

\bibitem{pcgrad}
Tianhe Yu, Saurabh Kumar, Abhishek Gupta, Sergey Levine, Karol Hausman, and
  Chelsea Finn.
\newblock Gradient surgery for multi-task learning.
\newblock In {\em Proceedings of the 33rd Advances in Neural Information
  Processing Systems}, 2020.

\bibitem{ZhangY21}
Yu~Zhang and Qiang Yang.
\newblock A survey on multi-task learning.
\newblock {\em {IEEE} Transactions on Knowledge and Data Engineering}, 2021.

\end{thebibliography}
\bibliographystyle{plain}
}

\newpage
\appendix


\section*{\centering {\Large Appendix}}

\section{Proof of the Mean Value $\mathbb{E}({\mathbf{\lambda}})$} \label{sec:mean}
Suppose that $\tilde{{\lambda}}_t (t=1,\cdots,T)$ are independent and identically distributed (i.i.d.) random variables sampled from the Uniform or standard Normal distributions and $f$ is the softmax function. Then we have ${\lambda}_t = \frac{\exp(\tilde{{\lambda}}_t)}{\sum_{m=1}^T\exp(\tilde{{\lambda}}_m)}$ and
$$\mathbb{E}({\lambda_i}) = \mathbb{E}[\exp(\tilde{{\lambda}}_i)]\mathbb{E}\bigg[\frac{1}{\sum_{m=1}^T\exp(\tilde{{\lambda}}_m)}\bigg]+ \mathrm{Cov}\left(\exp(\tilde{{\lambda}}_i), \frac{1}{\sum_{m=1}^T\exp(\tilde{{\lambda}}_m)}\right),$$
where $\mathrm{Cov}(\cdot, \cdot)$ denotes the covariance between two random variables.
Since $\{\tilde{{\lambda}}_t\}_{t=1}^T$ are i.i.d random variables, we have $\mathbb{E}[\exp(\tilde{{\lambda}}_i)]=\mathbb{E}[\exp(\tilde{{\lambda}}_j)]$ and $ \mathrm{Cov}(\exp(\tilde{{\lambda}}_i), 1/ \sum_{m=1}^T\exp(\tilde{{\lambda}}_m))= \mathrm{Cov}(\exp(\tilde{{\lambda}}_j), 1/ \sum_{m=1}^T\exp(\tilde{{\lambda}}_m))$.
Therefore, we obtain $$\mathbb{E}({\lambda}_i) = \mathbb{E}({\lambda}_j), \forall 1\le i, j \le T.$$
Moreover, we have $$\sum_{t=1}^T\mathbb{E}({\lambda}_t) = \sum_{t=1}^T\frac{\sum_{k=1}^K{\lambda}_{t}^k}{K}=\frac{\sum_{k=1}^K\sum_{t=1}^T{\lambda}_{t}^k}{K}=1.$$
Thus we have $\mathbb{E}(\bm{\lambda}) = (\frac{1}{T},\cdots,\frac{1}{T}).$
Similarly, we can prove the same result for the Bernoulli and c-Bernoulli distributions with the normalization function $f$ as $f(\tilde{\bm{\lambda}})=\tilde{\bm{\lambda}}/(\sum_{t=1}^T\tilde{\lambda}_t)$.

\section{Proof of Section \ref{sec:analysis}} \label{sec:analysis_appendix}
\subsection{Proof of Theorem \ref{thm:1}}
Suppose $\mathcal{L}_{\mathrm{RLW}}(\theta)=\bm{\lambda}^\top\bm{\ell}(\theta)$, where $\bm{\lambda}$ is a random variable sampled from a random distribution in every training iteration.

Since $\ell_t $ is $c_t$-strongly convex w.r.t. $\theta$,
for any two points $\theta_1$ and $\theta_2$ in $\mathbb{R}^d$, we have
\begin{align}
\label{eq:convex}
\left< \nabla \bm{\lambda}^\top \bm{\ell} (\theta_1) - \nabla \bm{\lambda}^\top \bm{\ell} (\theta_2), \theta_1 - \theta_2 \right> & = \sum_{t=1}^T \lambda_t \left< \nabla \ell_t(\theta_1) - \nabla \ell_t (\theta_2), \theta_1 - \theta_2 \right> \nonumber \\  &\ge \sum_{t=1}^T c_t\lambda_t \|\theta_1 - \theta_2 \|^2.
\end{align}
Since $0 \le \lambda_t \le 1$, we have $\sum_{t=1}^T c_t\lambda_t \ge c$, where $c = \min_{1 \le t \le T} \{c_t\}$. Then for any $\bm{\lambda}$, $\mathcal{L}_{\mathrm{RLW}}(\theta)$ is $c$-strongly convex.

With notations in Theorem \ref{thm:1}, we have
\begin{align*}
\|\theta_{k+1} - \theta_* \|^2 &= \|\theta_k - \theta_* - \eta \nabla\bm{{\lambda}}^\top\bm{\ell}(\tilde{\bm{\mathcal{D}}};\theta_k) \|^2 \\
&= \|\theta_k - \theta_* \|^2 - 2 \eta \left<\theta_k - \theta_*, \nabla\bm{{\lambda}}^\top\bm{\ell}(\tilde{\bm{\mathcal{D}}};\theta_k)\right> + \eta^2 \|\nabla\bm{{\lambda}}^\top\bm{\ell}(\tilde{\bm{\mathcal{D}}};\theta_k) \|^2.
\end{align*}
Note that $\mathbb{E}_{\bm{{\lambda}}} \left[\mathbb{E}_{\tilde{\bm{\mathcal{D}}}} [ \nabla \bm{{\lambda}}^\top\bm{\ell}(\tilde{\bm{\mathcal{D}}};\theta_k) ]\right] = \nabla \bm{\mu}^\top  \bm{\ell}(\bm{\mathcal{D}};\theta_k)$ and
\begin{align*}
\mathbb{E}_{\bm{{\lambda}}} \left[\mathbb{E}_{\tilde{\bm{\mathcal{D}}}} [ \|\nabla\bm{{\lambda}}^\top\bm{\ell}(\tilde{\bm{\mathcal{D}}};\theta_k) \|^2 ]\right]  &\le
\mathbb{E}_{\bm{{\lambda}}} \left[\mathbb{E}_{\tilde{\bm{\mathcal{D}}}} [ \|\bm{{\lambda}}^\top\|^2 \|\nabla\bm{\ell}(\tilde{\bm{\mathcal{D}}};\theta_k) \|^2 ]\right] \\
&\le  \mathbb{E}_{\bm{{\lambda}}}\bigg[\sum_{t=1}^T \lambda_t^2\bigg] \cdot \sum_{t=1}^T \sigma_t^2 \\
& \le\sum_{t=1}^T \sigma_t^2,
\end{align*}
where the first inequality is due to the Cauchy-Schwarz inequality and the third inequality is due to $0 \le \lambda_t \le 1$ and $\sum_t\lambda_t=1$.
Then, by defining $\kappa = \sum_{t=1}^T \sigma_t^2$, we obtain
\begin{align}
\label{p:1}
\mathbb{E}_{\bm{{\lambda}}} \left[\mathbb{E}_{\tilde{\bm{\mathcal{D}}}} [  \|\theta_{k+1} - \theta_* \|^2 ]\right] & \le \|\theta_k - \theta_* \|^2 - 2 \eta \left<\theta_k - \theta_*,\nabla \bm{\mu}^\top  \bm{\ell}(\theta_k)\right> +  \eta^2 \kappa \nonumber \\
& \le (1-2\eta c )  \|\theta_k - \theta_* \|^2 + \eta^2 \kappa.
\end{align}
If $1-2\eta c>0$, we recursively apply the inequality (\ref{p:1}) over the first $k$ iterations and we can obtain
\begin{align*}
\mathbb{E}[  \|\theta_{k+1} - \theta_* \|^2 ] & \le (1-2\eta c )^k  \|\theta_0 - \theta_* \|^2 + \sum_{j=0}^{k-1}(1-2\eta c )^j  \eta^2 \kappa \\
& \le (1-2\eta c )^k  \|\theta_0 - \theta_* \|^2 + \frac{\eta \kappa}{2c}.
\end{align*}
Thus the inequality (\ref{thm:1eq}) holds if $\eta \le \frac{1}{2c}$.

According to inequality (\ref{p:1}), the minimal value of a quadratic function
$ g_{\varepsilon}(\eta) = (1-2\eta c )\varepsilon + \eta^2 \kappa$ is achieved at $\eta_* = \frac{\varepsilon c}{\kappa}$. By setting $\|\theta_{0} - \theta_* \|^2 = \varepsilon_0$, we have
\begin{align*}
\mathbb{E}[  \|\theta_{k+1} - \theta_* \|^2 ] &\le g_{\|\theta_{k} - \theta_* \|^2}(\eta_*) \\
& = (1- \frac{2\|\theta_{k} - \theta_* \|^2 c^2}{\kappa})\|\theta_{k} - \theta_* \|^2 \\
& \le (1- \frac{2\varepsilon c^2}{\kappa}) \|\theta_{k} - \theta_* \|^2 \\
& \le (1- \frac{2\varepsilon c^2}{\kappa})^k \varepsilon_0.
\end{align*}
Then if $\mathbb{E}[  \|\theta_{k+1} - \theta_* \|^2 ] \ge \varepsilon$, we have $\varepsilon \le (1- \frac{2\varepsilon c^2}{\kappa})^k \varepsilon_0.$
Therefore, $k\le \frac{ \kappa}{2\varepsilon c^2} \log \left( \frac{\varepsilon_0}{\varepsilon} \right).$

\subsection{Proof of Theorem \ref{thm:2}}
Since $\varphi_k = \theta_k - \eta \nabla \bm{\mu}^\top\bm{\ell}(\theta_k)$ and $\theta_{k+1} = \theta_{k} - \eta(\nabla \bm{\mu}^\top\bm{\ell}(\theta_k) +\xi_k)$, we have $$\varphi_{k+1} = \varphi_k - \eta \xi_k - \nabla \bm{\mu}^\top\bm{\ell}(\varphi_k - \eta \xi_k ).$$
Since the loss function $\ell_t(\theta)$ of task $t$ is $c_t$-one point strongly convex w.r.t. a given point $\theta_*$ after convolved with noise $\xi$, similar to inequality (\ref{eq:convex}), we have
$$\left< \nabla \mathbb{E}_{\xi}[\bm{\mu}^\top\bm{\ell}(\varphi - \eta \xi)],\varphi - \theta_* \right> \ge c \|\varphi- \theta_* \|^2,$$
where $c = \min_{1 \le t \le T} \{c_t\}$. Since $\nabla \ell_t(\theta)$ is $M_t$-Lipschitz continuous, for any two points $\theta_1$ and $\theta_2$ in $\mathbb{R}^d$, we have
\begin{equation}
\label{eq:lips}
\| \nabla \bm{\mu}^\top \bm{\ell} (\theta_1) - \nabla \bm{\mu}^\top \bm{\ell} (\theta_2)\|  = \sum_{t=1}^T \mu_t \| \nabla \ell_t(\theta_1) - \nabla \ell_t (\theta_2)\| \le \sum_{t=1}^T M_t\mu_t \|\theta_1 - \theta_2 \|.
\end{equation}
Note that $\sum_{t=1}^T M_t\mu_t \le M$, where $M = \max_{1 \le t \le T} \{M_t\}$. Therefore,
$\nabla \bm{\mu}^\top\bm{\ell}(\theta)$ is $M$-Lipschitz continuous. Then we can get
\begin{align*}
\mathbb{E}[\| \varphi_{k+1} - \theta_*\|^2] &=   \mathbb{E}[\| \varphi_k - \eta \xi_k - \nabla \bm{\mu}^\top\bm{\ell}(\varphi_k - \eta \xi_k ) - \theta_*\|^2] \\
& \le \mathbb{E}[\| \varphi_k- \theta_*\|^2 +\| \eta \xi_k \|^2 +\| \nabla \bm{\mu}^\top\bm{\ell}(\varphi_k - \eta \xi_k )\|^2  - 2\left<\varphi_k- \theta_*,\eta \xi_k  \right>  \\
& \ ~~~~~~  - 2\left<\varphi_k- \theta_*,  \nabla \bm{\mu}^\top\bm{\ell}(\varphi_k - \eta \xi_k ) \right>+ 2\left< \nabla \bm{\mu}^\top \bm{\ell}(\varphi_k - \eta \xi_k ),  \eta \xi_k \right>]    \\
& \le \| \varphi_k- \theta_*\|^2 + \eta^2r^2 + \mathbb{E}[\| \nabla \bm{\mu}^\top\bm{\ell}(\varphi_k - \eta \xi_k )\|^2] - 2\eta c \| \varphi_k- \theta_*\|^2 \\
& \ ~~~~~~ + 2\mathbb{E}[\left< \nabla \bm{\mu}^\top\bm{\ell}(\varphi_k - \eta \xi_k ) -\nabla \bm{\mu}^\top\bm{\ell}(\varphi_k),  \eta \xi_k \right>]\\
& \le (1-2\eta c)\|\varphi_{k} - \theta_*\|^2 + \eta^2r^2 + \eta^2 \mathbb{E}[\|M(\theta_* - (\varphi_k - \eta \xi_k))\|^2 ] + 2 \eta^3r^2M \\
& \le (1-2\eta c)\|\varphi_{k} - \theta_*\|^2 + \eta^2r^2 + \eta^2M^2 \|\varphi_{k} - \theta_*\|^2 +  \mathbb{E}[\left< \varphi_{k} - \theta_*, \eta\xi_k\right> ] \\
& \ ~~~~~~ + \eta^2M^2\mathbb{E}[\| \eta \xi_k \|^2] + 2 \eta^3r^2M  \\
& \le (1-2\eta c + \eta^2M^2) \|\varphi_{k} - \theta_*\|^2 + \eta^2r^2(1+\eta M)^2,
\end{align*}
where the second inequality is due to the convexity assumption and $\mathbb{E}[\xi_k]=0$, the third and forth  inequalities are due to the Lipschitz continuity. We set $\rho = 2\eta c - \eta^2M^2$ and $\beta = \eta^2r^2(1+\eta M)^2$. If $\rho \ge 0$, we have $\eta \le \frac{c}{M^2}$, then we get
\begin{align*}
\mathbb{E}[\| \varphi_{k+1} - \theta_*\|^2] & \le (1-\rho) \|\varphi_{k} - \theta_*\|^2 + \beta \\
& \le (1-\rho)^k \|\varphi_0 - \theta_*\|^2 + \sum_{j=0}^{k-1}(1-\rho)^j \beta \\
& \le (1-\rho)^k \|\varphi_0 - \theta_*\|^2 + \frac{\beta}{\rho}.
\end{align*}
So if $K \le \frac{1}{\rho}\log \left(\frac{\rho\varepsilon_0}{\beta}\right)$, we have $\mathbb{E}[\| \varphi_{K+1} - \theta_*\|^2] \le \frac{2\beta}{\rho}.$ Then by the Markov inequality, with probability at least $1-\delta$, we have
$$ \|\varphi_K - \theta_* \|^2 \le  \frac{2\beta}{\rho\delta}.$$

\subsection{Noise Upper Bound}\label{appendix:noise}

Suppose the noise produced by the EW method is $\bar{\xi} = \| \nabla \bm{{\mu}}^\top\bm{\ell}(\tilde{\bm{\mathcal{D}}};\theta) - \nabla \bm{{\mu}}^\top\bm{\ell}(\bm{\mathcal{D}};\theta) \|$ and $\|\bar{\xi}\|^2 \le R$. The noise produced by the RLW method is $\xi = \| \nabla \bm{{\lambda}}^\top\bm{\ell}(\tilde{\bm{\mathcal{D}}};\theta) - \nabla \bm{{\mu}}^\top\bm{\ell}(\bm{\mathcal{D}};\theta) \|$. We have
\begin{align*}
\|\xi\|^2 &= \| \nabla \bm{{\lambda}}^\top\bm{\ell}(\tilde{\bm{\mathcal{D}}};\theta) - \nabla \bm{{\mu}}^\top\bm{\ell}(\tilde{\bm{\mathcal{D}}};\theta) + \nabla \bm{{\mu}}^\top\bm{\ell}(\tilde{\bm{\mathcal{D}}};\theta)- \nabla \bm{{\mu}}^\top\bm{\ell}(\bm{\mathcal{D}};\theta) \|^2 \\
& = \|(\bm{{\lambda}}^\top -\bm{{\mu}}^\top) \nabla\bm{\ell}(\tilde{\bm{\mathcal{D}}};\theta)  \|^2 +2\left<  (\bm{{\lambda}}^\top -\bm{{\mu}}^\top)\bm{\ell}(\tilde{\bm{\mathcal{D}}};\theta),\bar{\xi} \right> + \|\bar{\xi}\|^2.
\end{align*}
Because the noise $\bar{\xi}$ can be any direction, there exists a constant $s >0$ such that $\|\bar{\xi}\|^2 = R$ and $\bar{\xi} = s(\bm{{\lambda}}^\top -\bm{{\mu}}^\top)\nabla\bm{\ell}(\tilde{\bm{\mathcal{D}}};\theta)$. Then, we have $\|\xi\|^2 \le (1+2s) \|\bm{{\lambda}} -\bm{{\mu}}\|^2\|\nabla \bm{\ell}(\tilde{\bm{\mathcal{D}}};\theta)\|^2 + R$. Thus, the norm of the noise provided by the RLW method has a larger supremum than EW.

\section{Additional Details about the XTREME Benchmark}

\begin{table}[!h]
\centering
\caption{The numbers of training, validation, and test data for each language in PI and POS problems from the XTREME benchmark.}
\label{tab:nlp_data_count}
\begin{tabular}{ccc}
\toprule
& \textbf{PI} & \textbf{POS} \\
\midrule
\texttt{en} & 49.4K+2.0K+2.0K & 6.9K+1.8K+3.2K \\
\texttt{zh} & 49.4K+2.0K+2.0K & 4.0K+0.5K+2.9K \\
\texttt{de} & 49.4K+2.0K+2.0K & - \\
\texttt{es} & 49.4K+2.0K+2.0K & - \\
\texttt{te} & - & 1.0K+0.1K+0.1K \\
\texttt{vi} & - & 1.4K+0.8K+0.8K \\
\bottomrule
\end{tabular}
\end{table}

\section{Additional Experimental Results} \label{sec:additional_results}

\subsection{Results on the CityScapes Dataset}

\paragraph{Dataset.} The {CityScapes} dataset \cite{CordtsORREBFRS16} is a large-scale urban street scene understanding dataset and it is comprised of a diverse set of stereo video sequences recorded from 50 different cities in fine weather during the daytime. It contains 2,975 and 500 annotated images for training and test, respectively. This dataset includes two tasks: 7-class semantic segmentation and depth estimation.

\paragraph{Implementation details.} For the {CityScapes} dataset, the network architecture and optimizer are the same as those in the NYUv2 dataset. We resize all the images to $128\times 256$ and set the batch size to 64 for training. We use the cross-entropy loss and $L_1$ loss for the semantic segmentation and depth estimation tasks, respectively.

\paragraph{Results.} The results on the CityScapes dataset are shown in Table~\ref{tbl:mtl-cityscapes}. The empirical observations are similar to those on the NYUv2 dataset in Table~\ref{tbl:mtl-nyu_dmtl}. Firstly, both the RLW and RGW strategies significantly outperform the EW method. Secondly, the RLW method can outperform most of the loss balancing baselines except the IMTL-L method. Moreover, the RGW method achieves 2.36\% performance improvement and outperforms all of the baselines.


\begin{table}[!h]
\centering
\caption{Performance on the \textbf{CityScapes} dataset with two tasks: 7-class semantic segmentation and depth estimation. The best results for each task on each measure over loss/gradient balancing methods are marked with superscript $*$/$\dag$. The best results for each task on each measure are highlighted in \textbf{bold}. $\uparrow$ ($\downarrow$) means the higher (lower) the result, the better the performance.}
\label{tbl:mtl-cityscapes}
\begin{tabular}{ccccccc}
\toprule
\multicolumn{2}{c}{\multirow{2.5}{*}{\textbf{Methods}}} & \multicolumn{2}{c}{\textbf{Segmentation}} & \multicolumn{2}{c}{\textbf{Depth}}&
\multirow{2.5}{*}{\bm{$\Delta_{\mathrm{p}}$}${\uparrow}$}\\
\cmidrule(lr){3-4} \cmidrule(lr){5-6}
& & \textbf{mIoU${\uparrow}$} &  \textbf{Pix Acc$\uparrow$} &  \textbf{Abs Err$\downarrow$} & \textbf{Rel Err$\downarrow$} \\
\midrule
& EW  & 68.71 & 91.50 & 0.0132 & 45.58 & +0.00\%\\
\midrule
\multirow{5}{*}{\rotatebox{90}{Loss Bal.}} & UW & 68.84 & 91.53 & 0.0132 & 46.18 & -0.09\%\\
& DWA & 68.56 & 91.48 & 0.0135 & 44.49 & +0.05\% \\
& IMTL-L & \textbf{69.71}$^*$ & 91.77$^*$ & 0.0128$^*$ & 45.08 & +1.58\%$^*$ \\
& MOML & 69.34 & {91.65} & 0.0129 & 46.33 & +0.59\% \\
& \textbf{RLW (ours)} & 68.78 & 91.45 & 0.0134 & \textbf{43.68}$^*$ & +0.69\% \\
\midrule
\multirow{9}{*}{\rotatebox{90}{Gradient Bal.}} & MGDA-UB & 68.41 & 91.13 & \textbf{0.0124}$^\dag$ & 46.85 & +0.64\% \\
& GradNorm & 68.60 & 91.48 & 0.0133 & 45.32 & +0.01\% \\
& PCGrad & 68.54 & 91.47 & 0.0135 & 44.82 & -0.10\%\\
& GradDrop & 68.62 & 91.45 & 0.0136 & 45.05 & -0.42\%\\
& IMTL-G & 68.62 & 91.48 & 0.0130 & 44.29 & +1.09\%\\
& GradVac & 68.60 & 91.47 & 0.0134 & 44.92 & -0.06\%\\
& CAGrad & 68.89 & 91.50 & 0.0128 & 44.72 & +1.38\%\\
& RotoGrad & 68.96 & 91.47 & 0.0127 & 43.85$^\dag$ & +2.13\%\\
& \textbf{RGW (ours)} & 69.68$^\dag$ & \textbf{91.85}$^\dag$ & 0.0127 & 43.91 & \textbf{+2.36\%}$^\dag$\\
\bottomrule
\end{tabular}
\end{table}

\subsection{Results on the CelebA Dataset}

\paragraph{Dataset.} The CelebA dataset \cite{liu2015faceattributes} is a large-scale face attributes dataset with 202,599 face images, each of which has 40 attribute annotations. It is split into three parts: 162,770, 19,867, and 19,962 images for training, validation, and testing, respectively. Hence, this dataset contains 40 tasks and each task is a binary classification problem for one attribute.

\begin{table}[!h]
\centering
\caption{Average classification accuracy (\%) of different methods on the \textbf{CelebA} dataset with forty tasks. The best results over loss/gradient balancing methods are marked with superscript $*$/$\dag$. The best results are highlighted in \textbf{bold}. }
\label{tbl:mtl-celeba}
\begin{tabular}{ccc}
\toprule
\multicolumn{2}{c}{\multirow{2}{*}{\textbf{Methods}}} & \multirow{2}{*}{\textbf{Avg Acc}}\\
& & \\
\midrule
& EW & 90.70 \\
\midrule
\multirow{5}{*}{\rotatebox{90}{Loss Bal.}} & UW & 90.84 \\
& DWA & 90.77\\
& IMTL-L & 90.46\\
& MOML & \textbf{90.94}$^*$\\
& \textbf{RLW (ours)} & 90.73\\
\midrule
\multirow{9}{*}{\rotatebox{90}{Gradient Bal.}} & MGDA-UB & 90.40\\
& GradNorm & 90.77 \\
& PCGrad & 90.85$^\dag$ \\
& GradDrop & 90.71 \\
& IMTL-G & 90.80\\
& GradVac & 90.75\\
& CAGrad & 90.72 \\
& RotoGrad & 90.45\\
& \textbf{RGW (ours)} & 90.00\\
\bottomrule
\end{tabular}
\end{table}

\paragraph{Implementation details.} We use the ResNet-18 network as a shared feature extractor and a fully connected layer with two output units as a task-specific head for each task. All the images are resized to $64\times 64$. The Adam optimizer with the learning rate as $10^{-3}$ is used for training and the batch size is set to 512. The cross-entropy loss is used for the 40 tasks.

\paragraph{Results.} Since the number of tasks in the CelebA dataset is large, we only report the average classification accuracy on the forty tasks in Table~\ref{tbl:mtl-celeba}. According to the results, the proposed RLW strategy slightly outperforms the EW method and performs comparably with loss balancing baseline methods. However, we can find that the RGW method and most of the gradient balancing methods are worse or achieve very limited improvement over the EW method, which indicates the gradient weighting is not suitable for the CelebA dataset.


\subsection{Results on the Office-31 and Office-Home Datasets}

\paragraph{Datasets.} The Office-31 dataset \cite{saenko2010adapting} consists of three domains: Amazon (\textbf{A}), DSLR (\textbf{D}), and Webcam (\textbf{W}), where each domain contains 31 object categories, and it contains 4,110 labeled images. We randomly split the whole dataset with 60\% for training, 20\% for validation, and the rest 20\% for testing.
The Office-Home dataset \cite{venkateswara2017deep} has four domains: artistic images (\textbf{Ar}), clip art (\textbf{Cl}), product images (\textbf{Pr}), and real-world images (\textbf{Rw}). It has 15,500 labeled images in total and each domain contains 65 classes. We make the same split as the Office-31 dataset. For both datasets, we consider the multi-class classification problem on each domain as a task. Similar to multilingual problems from the XTREME benchmark, each task in both Office-31 and Office-Home datasets has its own input data.

\paragraph{Implementation details.} We use the same configuration for the Office-31 and Office-Home datasets. Specifically, the ResNet-18 network pre-trained on the ImageNet dataset is used as a shared backbone among tasks and a fully connected layer is applied as a task-specific output layer for each task. All the input images are resized to $224\times 224$. We use the Adam optimizer with the learning rate as $10^{-4}$ and the weight decay as $10^{-5}$ and set the batch size to 128 for training. The cross-entropy loss is used for all tasks in both datasets.

\begin{table}[!h]
\caption{Classification accuracy (\%) of different methods on the \textbf{Office-31} and \textbf{Office-Home} datasets. The best results for each domain over loss/gradient balancing methods are marked with superscript $*$/$\dag$. The best results for each task are highlighted in \textbf{bold}. }
\label{tbl:mtl-officehome}
\centering
\resizebox{\linewidth}{!}{
\begin{tabular}{cc|cccc|ccccc}
\toprule
\multicolumn{2}{c|}{\multirow{2.5}{*}{\textbf{Methods}}} & \multicolumn{4}{c|}{\textbf{Office-31}} & \multicolumn{5}{c}{\textbf{Office-Home}} \\
\cmidrule(r){3-6} \cmidrule(l){7-11}
&&{\textbf{A}} & {\textbf{D}} & {\textbf{W}} & {\textbf{Avg}} &{\textbf{Ar}}&{\textbf{Cl}}&{\textbf{Pr}}&{\textbf{Rw}}&{\textbf{Avg}}\\
\midrule
& EW & 82.73 & 96.72 & 96.11 & 91.85 & 62.99 & 76.48 & 88.45 & 77.72 & 76.41\\
\midrule
\multirow{5}{*}{\rotatebox{90}{Loss Bal.}} & UW & 82.73 & 96.72$^*$ & 95.55 & 91.66 & 63.94 & 75.62 & 88.55 & 78.05 & 76.54\\
& DWA & 82.22 & 96.72$^*$ & 96.11 & 91.68 & 63.37 & 76.05 & 89.08 & 77.62 & 76.53\\
& IMTL-L & 83.76 & 96.72$^*$ & 95.55 & 92.01 & \textbf{65.46}$^*$ & \textbf{79.08}$^*$ & 88.45 & 78.81 & 77.95$^*$\\
& MOML & \textbf{84.78}$^*$ & 95.08 & 96.67$^*$ & 92.17 & 64.70 & 77.03 & 88.24 & 80.00 & 77.49\\
& \textbf{RLW (ours)} & 83.76 & 96.72$^*$ & 96.67$^*$ & {92.38}$^*$ & 62.80 & 76.48 & \textbf{90.57}$^*$ & \textbf{80.21}$^*$ & {77.52}\\
\midrule
\multirow{9}{*}{\rotatebox{90}{Gradient Bal.}} & MGDA-UB  & 81.02 & 95.90 & \textbf{97.77}$^\dag$ & 91.56 & 64.32 & 75.29 & 89.72 & 79.35 & 77.17\\
& GradNorm & 83.93 & \textbf{97.54}$^\dag$ & 94.44 & 91.97 & \textbf{65.46}$^\dag$ & 75.29 & 88.66 & 78.91 & 77.08\\
& PCGrad & 82.22 & 96.72 & 95.55 & 91.49 & 63.94 & 76.05 & 88.87 & 78.27 & 76.78\\
& GradDrop & 84.27$^\dag$ & 95.08 & 96.11 & 91.82 & 64.70 & 77.03 & 88.02 & 79.13 & 77.22\\
& IMTL-G & 82.22 & 95.90 & 96.11 & 91.41 & 63.37 & 76.05 & 89.19 & 79.24 & 76.96\\
& GradVac & 82.73 & \textbf{97.54}$^\dag$ & 95.55 & 91.94 & 63.18 & 76.48 & 88.66 & 77.83 & 76.53\\
& CAGrad & 82.22 & 96.72 & 96.67 & 91.87 & 63.75 & 75.94 & 89.08 & 78.27 & 76.75\\
& RotoGrad & 82.90 & 96.72 & 96.11 & 91.91 & 61.85 & 77.03 & 90.36$^\dag$ & 78.59 & 76.95\\
& \textbf{RGW (ours)} & 84.27$^\dag$ & 96.72 & 96.67 & \textbf{92.55}$^\dag$ & 65.08 & 78.65$^\dag$ & 88.66 & 79.89$^\dag$ & \textbf{78.07}$^\dag$\\
\bottomrule
\end{tabular}}
\end{table}

\paragraph{Results.} According to the results shown in Table~\ref{tbl:mtl-officehome}, we can see both the RLW and RGW strategies outperform the EW method on both two datasets in terms of the average classification accuracy over tasks, which implies the effectiveness of the RW methods. Moreover, the RGW method achieves the best performance (92.55\% and 78.07\% in term of the average accuracy) over all baselines on the Office-31 and Office-Home datasets, respectively.


\end{document}